\documentclass{article}

\PassOptionsToPackage{numbers, sort&compress}{natbib}

    \usepackage[preprint]{neurips_2025}

\usepackage[utf8]{inputenc} 
\usepackage[T1]{fontenc}    
\usepackage{hyperref}       
\usepackage{url}            
\usepackage{booktabs}       
\usepackage{amsfonts}       
\usepackage{nicefrac}       
\usepackage{microtype}      
\usepackage{xcolor}         
\usepackage{amsmath}
\usepackage{amsthm}
\usepackage{algorithm}
\usepackage[noend]{algpseudocode}
\usepackage{adjustbox}
\usepackage{cleveref}
\usepackage{subcaption}
\usepackage{multirow}
\usepackage{booktabs}
\crefname{theorem}{Theorem}{Theorems}
\Crefname{theorem}{Theorem}{Theorems}

\newtheorem{theorem}{Theorem}[section]
\crefname{definition}{Definition}{Definitions}
\Crefname{definition}{Definition}{Definitions}
\newtheorem{definition}{Definition}[section]

\title{Formal Synthesis of Certifiably Robust Neural Lyapunov-Barrier Certificates}

\author{%
  Chengxiao Wang \\
  University of Illinois, \\ Urbana-Champaign \\
  \texttt{cw124@illinois.edu} \\
  \And
  Haoze Wu \\
  Amherst College \\
  \texttt{hwu@amherst.edu}
  \And
  Gagandeep Singh \\
  University of Illinois, \\ Urbana-Champaign \\
  \texttt{ggnds@illinois.edu} \\
}

\begin{document}

\maketitle

\begin{abstract}
Neural Lyapunov and barrier certificates have recently been used as powerful tools for verifying the safety and stability properties of deep reinforcement learning (RL) controllers. However, existing methods offer guarantees only under fixed ideal unperturbed dynamics, limiting their reliability in real-world applications where dynamics may deviate due to uncertainties. In this work, we study the problem of synthesizing \emph{robust neural Lyapunov barrier certificates} that maintain their guarantees under perturbations in system dynamics. We formally define a robust Lyapunov barrier function and specify sufficient conditions based on Lipschitz continuity that ensure robustness against bounded perturbations. We propose practical training objectives that enforce these conditions via adversarial training, Lipschitz neighborhood bound, and global Lipschitz regularization. We validate our approach in two practically relevant environments, Inverted Pendulum and 2D Docking. The former is a widely studied benchmark, while the latter is a safety-critical task in autonomous systems. We show that our methods significantly improve both certified robustness bounds (up to $4.6$ times) and empirical success rates under strong perturbations (up to $2.4$ times) compared to the baseline. Our results demonstrate effectiveness of training robust neural certificates for safe RL under perturbations in dynamics.
\end{abstract}

\section{Introduction}
Deep Reinforcement learning (RL) has demonstrated remarkable performance in a wide range of sequential decision-making tasks, such as robotic control ~\cite{kalashnikov2018scalable}, autonomous driving ~\cite{sallab2017deep}, game playing ~\cite{mnih2013playing}. However, using RL controllers in safety-critical domains such as autonomous driving remains challenging due to the lack of formal guarantees on safety and stability. Classical control theory provides formal guarantees through analytic Lyapunov and barrier functions, but deriving such certificates for 
the complex dynamics of modern DRL tasks is notoriously difficult. To sidestep this bottleneck, recent work has explored parameterizing certificates with neural networks. These methods learn a neural certificate function jointly with the policy, and prove the validity of the learned certificate with formal verifiers~\cite{chang2019neural,abate2021fossil,abate2020formal,mandal2024formally}. 

By and large, the validity of the neural certificate is only established for a fixed ideal dynamics in existing work. However, in practice, the actual observed next state by the controller might deviate from the ideal dynamics, due to modeling errors ~\cite{pinto2017robust}, actuation noise ~\cite{pan2019risk}, or environmental perturbations ~\cite{mandlekar2017adversarially,pinto2017robust}. The existence of these unavoidable uncertainty means that the controller may violate the safety or stability properties even if it has been fully verified against the ideal dynamics. Indeed, such sim-to-real-error has long been recognized and methods for guaranteeing robustness against bounded disturbances have been studied for analytic Lyapunov certificates~\cite{Cardona2025,freeman1994robust,garg2021robust}. However, analogous methods that guarantee such robustness property for \emph{neural} certificates are still largely unexplored. 

In this work, we bridge this gap and study the problem of synthesizing neural certificates robust against norm-bounded time-varying uncertainty. Unlike existing approach that obtains probabilistic guarantees for stochastic systems~\cite{lechner2022stability,vzikelic2023learning}, we aim to give deterministic guarantee against bounded perturbations in the system dynamics. We formally define a robust neural Lyapunov-barrier certificate and develop a training framework to obtain such certificates. The key idea is to enforce a stronger decrease condition on the Lyapunov function: instead of requiring decrease only at the next state, we ensure that the Lyapunov value decreases across a neighborhood of perturbed next states, so that it keeps the safety and liveness property even under perturbation. We formally establish the relationship between the Lipchitz constant of the certificate and the strengthened decrease condition, and, based on this theoretical result, investigate several candidate loss functions for training the robust certificate. We evaluated our approach in two case studies, and demonstrated that our robust training approaches, including adversarial training, and Lipschitz regularization, improve both certified robustness bound and empirical success rate under perturbations, by up to 4.6 times and 2.4 times over the baseline method respectively. 

To summarize, our contributions include:
\begin{itemize}
    \item We study, for the first time, the problem of synthesizing neural Lyapunov/barrier certificates that provably satisfy the \textbf{robust Lyapunov condition}, which guarantees the Lyapunov decrease condition under bounded perturbations in the system dynamic.
    \item We establish sufficient theoretical conditions for robustness that leverage Lipschitz continuity, and develop multiple practical training methods to enforce robustness, including global Lipschitz regularization, local neighborhood bounds, and adversarial training.
    \item We demonstrate the practical benefits of our approach in two case studies. Our approach improves the certified perturbation bounds by up to $4.6$ times, and for empirical robustness, we achieve up to $2.4$ times success rate compared to baselines under strong perturbations.
    
\end{itemize}

\section{Related Work}
In this section, we review related work on control certificates. A detailed discussion of other related work can be found in Appendix~\ref{appendix:related-work}. There are extensive works on safe control, most of which using Lyapunov certificate to guarantee the liveness and use barrier function to guarantee the safety, with traditional approaches including numerical methods, polynomial optimization, and simulation-guided synthesis ~\cite{giesl2015review}. 
Recent work has explored jointly training neural controllers and certificates using sampled data~\cite{dawson2022safe,liu2022safe}, and verifying trained certificates with formal tools~\cite{chang2019neural,abate2021fossil,abate2020formal,mandal2024formally,mathiesen2022safety}.
However, they only consider the dynamics that are deterministic without any perturbations. There has also been work on generating neural certificate for stochastic control system~\cite{lechner2022stability,vzikelic2023learning}, which provides probabilistic reach-avoid guarantees. In contrast, our work provides deterministic reach-avoid guarantee under norm-bounded time-varying perturbations. Previous work in neural certificate training has also considered enforcing Lipschitz continuity as a regularization term~\cite{lechner2022stability,vzikelic2023learning,dawson2022safe} in order to improving the efficiecny and the training and verification. In contrast, we leverage Lipchitz continuity to improve the robustness of the certificate against perturbations in the system dynamics.

\section{Background}

\subsection{Notation}
We denote the \(l_p\) norm of a vector \( x \) as \( \| x\|_p\), the spectral norm of matrix \(A\) as \( \| A \|_2 \). The \(l_p\) ball of radius \(\delta\) around \(x\) is \( \mathcal{B}_{\delta,p}(x) = \{x' : \|x' - x\|_p \leq \delta\} \). \(\mathcal{X}\) is the set of possible states and \(\mathcal{U}\) is the set of possible control inputs. The function \( f:  \mathcal{X} \times \mathcal{U} \rightarrow \mathcal{X} \) is a transition function that defines the system's transition dynamics. Given a current state \(x\) and a control input \(u\), the transition function outputs the next state. A policy \(\pi: \mathcal{X} \rightarrow \mathcal{U}\) maps each state to a control action and \( u = \pi(x) \) is the control action taken by policy \( \pi \) at state \(x\).

\subsection{Reach-While-Avoid task}

We consider a discrete-time dynamical system, which can be either linear or non-linear, defined as:
\begin{equation}
\label{eq:dynamic-system}
    x_{t+1} = f(x_t, u_t), \quad x_t \in \mathcal{X} \subset \mathbb{R}^n, \ u_t \in \mathcal{U} \subset \mathbb{R}^m,
\end{equation}
where \(x_t\) denotes the system state at time \(t\), \(u_t\) is the control input. 
We define three important subsets of the state space: the \textbf{unsafe set} \( \mathcal{X}_U \subset \mathcal{X} \), which must be avoided; the \textbf{goal set} \( \mathcal{X}_G \subset \mathcal{X} \setminus \mathcal{X}_U \), which the system must eventually reach; and the \textbf{initial set} \( \mathcal{X}_I \subset \mathcal{X} \setminus \mathcal{X}_U \), where the system starts.

The objective of \emph{Reach-While-Avoid} (RWA) task is to synthesize a policy \( \pi \) such that, for every \( x_0 \in \mathcal{X}_I \), the induced trajectory \( \{x_t\}_{t \geq 0} \) satisfies:
\[
\exists \tau \geq 0,\  x_\tau \in \mathcal{X}_G, \quad \land \quad \forall t \in [0, \tau],\ x_t \notin \mathcal{X}_U.
\]

\subsection{Control Lyapunov Barrier Function}
\label{definition:clbf}
A \emph{Control Lyapunov function} (CLF) represents the energy level of a state. As time passes, energy decreases until the system reaches the zero-energy equilibrium point ~\cite{haddad2008nonlinear}. A \emph{Control Barrier function} (CBF) ~\cite{ames2019control}, also an energy-based function, is used to certify that the system will never enter an unsafe state. This is done by assigning unsafe states' function values above a threshold and proving that the system's function value will never exceed the threshold.

A function \( V : \mathcal{X} \rightarrow \mathbb{R} \) is a \emph{Control Lyapunov Barrier Function }(CLBF) if it combines the properties of CLF and CBF. More formally, it can be defined as follows:
\begin{definition}[Control Lyapunov Barrier Function]
\label{def:lyapunov-barrier-certificate}

A function \( V : \mathcal{X} \rightarrow \mathbb{R}_{\geq c} \) is a Control Lyapunov Barrier Function with parameters \((\alpha, \beta, \epsilon)\) w.r.t. sets \( \mathcal{X}_I \) (initial), \( \mathcal{X}_G \) (goal), and \( \mathcal{X}_U \) (unsafe), if \(\alpha > \beta > c, ~ \epsilon > 0 \) and the following conditions hold:
\begin{align}
V(x) &\leq \beta, \quad \forall x \in \mathcal{X}_I, \label{eq:init-cond} \\
V(x) - V(f(x, \pi(x))) & \geq \epsilon, \quad \forall x \in \left\{ x \in \mathcal{X} \setminus \mathcal{X}_G \mid V(x) \leq \beta \right\}, \label{eq:decrease-cond} \\
V(x) &\geq \alpha, \quad \forall x \in \mathcal{X}_U. \label{eq:safety-cond}
\end{align}
\end{definition}
Intuitively, this definition encodes both safety and liveness properties. For safety property, since Eq.~\eqref{eq:init-cond} ensures that trajectories starting within a safe region and Eq.~\eqref{eq:decrease-cond} requires the certificate value to decrease at each step, the trajectory never enters the unsafe region defined by Eq.~\eqref{eq:safety-cond}. For liveness property(i.e., the reachability of the goal state), although no explicit constraint is added in \( V(x) \) for states in \( \mathcal{X}_G \), it is still guaranteed because without entering \( \mathcal{X}_G \), the certificate value would decrease infinitely by Eq.~\eqref{eq:decrease-cond}, contradicting the fact that \( V \) is lower-bounded. 

Naturally, this concept can be applied to the RWA task, where the sets \( \mathcal{X}_I \), \( \mathcal{X}_G \), and \( \mathcal{X}_U \) represent the initial, goal, and unsafe regions of the task, respectively. In this context, the CLBF for the RWA task is referred to as the \emph{Lyapunov Barrier Certificate} for the specific RWA task. These constraints ensure that any trajectory \( \{x_t\}_{t \geq 0} \) starting from \( x_0 \in \mathcal{X}_I \) given by Eq.~\eqref{eq:dynamic-system} stay away from \( \mathcal{X}_U \), and target toward \( \mathcal{X}_G \), while monotonically decreasing the certificate function \( V \).

\subsection{Global Lipschitz Bound}
Consider a feedforward neural network with \( \ell \) layers, mapping input \(x\) to output \( x_{\ell+1}\), defined by:
\[
x_1 = x, \quad x_{k+1} = \phi(W_k x_k + b_k) \text{ for } k = 1, \ldots, \ell-1, \quad x_{\ell+1} = W_{\ell} x_{\ell} + b_{\ell},
\]
where \( W_k \in \mathbb{R}^{n_{k+1} \times n_k} \), \( b_k \in \mathbb{R}^{n_{k+1}} \) are learned parameters, and \( \phi \) is the element-wise activation function.
A global Lipschitz constant \( L_p \) for a function \( f : \mathbb{R}^{n_1} \rightarrow \mathbb{R}^{n_{\ell+1}} \) with respect to a norm \( \| \cdot \|_p \) is defined as:
\[
L_p = \sup_{x, y \in \mathbb{R}^{n_1}, ~x \not= y}\frac{\|f(x) - f(y)\|_p}{\|x - y\|_p}
\]
The function is said to be Lipschitz continuous if \( L_p\) exists and is finite.
For activation functions with Lipschitz constant equal to 1, the global Lipschitz bound of the network with respect to the \( l_2 \)-norm can be upper-bounded by the product of the spectral norms of the weight matrix of each layer \cite{szegedy2013intriguing}:
\begin{equation}
\label{eq:lip-matrix-norm}
    L_2 \leq \prod_{k=1}^{\ell} \|W_k\|_2
\end{equation}
While this work focuses on the most common ReLU activations, our approach can extend to other popular activation functions that are 1-Lipschitz (e.g., Leaky ReLU, SoftPlus, Tanh, Sigmoid, ArcTan, or Softsign), as the same spectral norm-based Lipschitz bound applies.

\noindent\textbf{Extension to General Norms:}
Following the work in \cite{relations} we can upper bound the Lipschitz constant \( L_p\) with respect to \( l_p \) norm, by a scaled version of the \( l_2 \)-Lipschitz constant: 
\begin{equation}
\label{eq:lip-diff-norm}
    L_p \leq K_{p,2} \cdot L_2 ,
\end{equation}
Specifically, \( K_{p,2} = n^{\left(\frac{1}{2} - \frac{1}{p}\right)} m^{\left(\frac{1}{p} - \frac{1}{2}\right)} \) where \(n\) and \(m\) are the input and output dimensions of the neural network.
The detailed derivation is provided in Appendix~\ref{appendix:norm-inequality}.

\section{Methodology}
In Section~\ref{sec:robust-lyapunov-cond}, we introduce a robust extension of the standard CLBF condition, requiring the certificate to decrease not only at the next state but over its neighborhood, ensuring robustness to perturbations. In practice, this accounts for errors in the environment’s computed state. As reachable sets may grow exponentially, we require all to remain outside the unsafe region. Section~\ref{sec:training-loss-design} presents three loss function designs supporting the robust CLBF condition: an adversarial loss for empirical robustness, a local robustness loss using global Lipschitz bounds to conservatively upper-bound certificate values within a perturbed neighborhood, and a global Lipschitz regularization for certifiable robustness over the full state space. The overall training algorithm follows a counterexample-guided inductive synthesis (CEGIS) loop~\cite{solar2008program}. A full description of the algorithm is provided in Appendix~\ref{appendix:training-alg}.

\subsection{Robust Lyapunov Barrier Certificate}
\label{sec:robust-lyapunov-cond}
We first define what it means for a controller of an RWA task to be robust under state perturbations. Figure~\ref{fig:robust-rwa} provides an illustrative example. In an RWA task, the robust controller is required to reach the goal region \( \mathcal{X}_G \) from an initial state in \( \mathcal{X}_I \), while avoiding the unsafe region \( \mathcal{X}_U \), even when states computed by the environment at each time step are perturbed by a bounded amount. The definition below formalizes this requirement.

\begin{definition}[Robust Controller for Reach-While-Avoid (RWA) Task]

A controller for Reach-While-Avoid task is said to be \emph{robust under state perturbations} \( \delta \in \mathbb{R} \) in the \( l_p \) norm if, for any initial state \( x_0 \in \mathcal{X}_I \), any induced trajectory
\( \{x_t\}_{t=0}^{\tau} \) given by \(x_{t+1}=f(x_t,u_t)+\delta_t\) where \(f\) is the transition function, \(u_t\) is the control input, \(\|\delta_t\|_p \leq \delta\), satisfies:
\[
 x_0 \in \mathcal{X}_I, \quad \land \quad \exists \tau \geq 0,\  x_\tau \in \mathcal{X}_G, \quad \land \quad \forall t \in \left[0, \tau-2\right],\ \mathcal{B}_{\delta,p}(f(x_t,u_t)) \cap \mathcal{X}_U = \emptyset.
\]
\end{definition}

\begin{figure}[t]

    \centering
    \begin{subfigure}[t]{0.48\textwidth}
        \centering
        \includegraphics[width=0.7\textwidth]{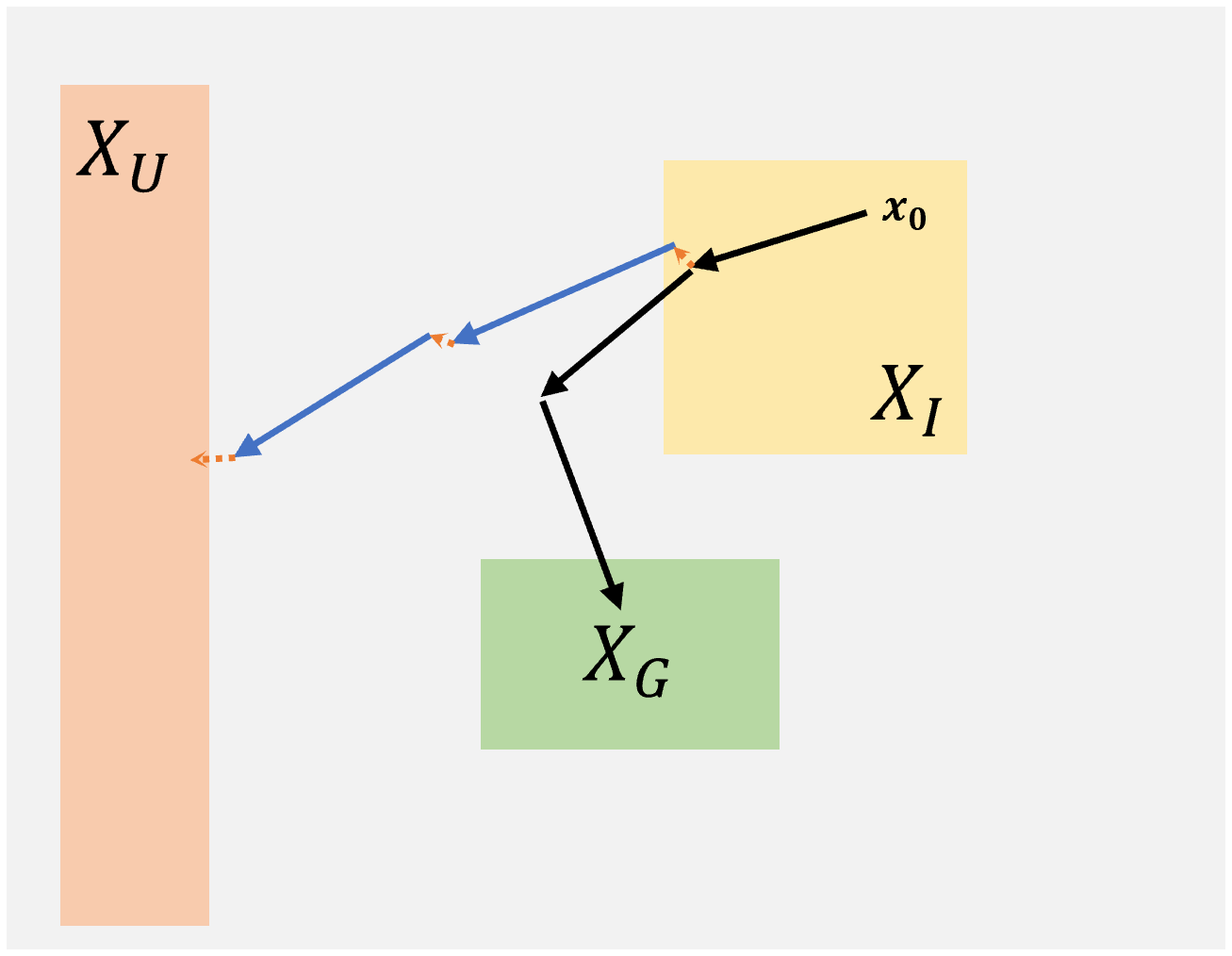}
        \caption{Standard CLBF-certified controller. The nominal trajectory reaches the goal \(X_G\), but after perturbations are added to each state, it may violate safety and enter the unsafe region \(X_U\).}
        \label{fig:nominal}
    \end{subfigure}
    \hfill
    \begin{subfigure}[t]{0.48\textwidth}
        \centering
        \includegraphics[width=0.7\textwidth]{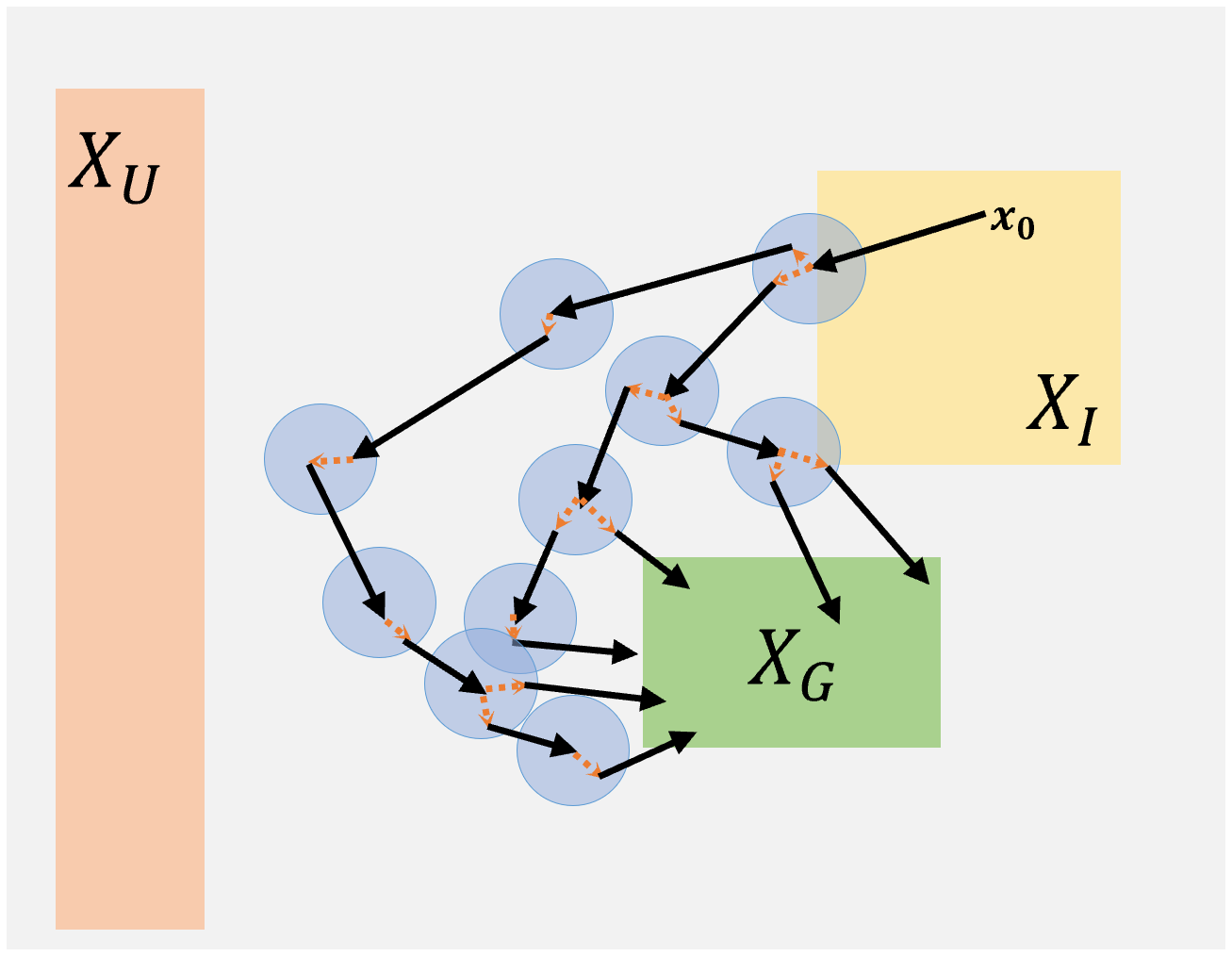}
        \caption{Robust CLBF-certified controller. At each step, the controller considers a full set of reachable next states within a perturbation neighborhood. All perturbed trajectories reach the goal region \(X_G\).} 
        \label{fig:robust}
    \end{subfigure}
    \caption{Comparison of standard vs. robust certified controllers under bounded state perturbations.}
    \label{fig:robust-rwa}
    \vspace{-0.3cm}
\end{figure}

Unlike traditional controllers for RWA tasks, our robust controller definition requires that, starting from any state reachable from the initial state (accounting for perturbations), the set of all possible next states must remain within the safe region. Although the controller output at each step determines a specific next state, we consider the entire neighborhood around it that state perturbations might induce. At each subsequent step, any state within the current reachable set may transition into another perturbed region, leading to a state expansion. 

To certify the controller satisfies this definition, we extend the classical Lyapunov Barrier Certificate to account for perturbations in states. Specifically, we modify the decrease condition Eq.~\eqref{eq:decrease-cond} to ensure that the certificate value decreases from \(x_t\) to \(x_{t+1}\) even under perturbation in \( x_{t+1}\). More formally, we define the \emph{robust Lyapunov Barrier Certificate} as follows:

\begin{definition}[Robust Lyapunov Barrier Certificate for RWA tasks]
\label{def:robust-lyapunov-barrier-certificate}
A function \( V : \mathcal{X} \rightarrow \mathbb{R}_{\geq c} \) is a \((\delta,p)\)-Robust Lyapunov Barrier Certificate with parameters \((\alpha, \beta, \epsilon)\) for a RWA task w.r.t. sets \( \mathcal{X}_I \) (initial), \( \mathcal{X}_G \) (goal), and \( \mathcal{X}_U \) (unsafe), if \(\alpha > \beta > c, ~ \epsilon > 0 \) and the following conditions hold:
\begin{align}
V(x) &\leq \beta, \quad \forall x \in \mathcal{X}_I, \label{eq:robust-init-cond} \\
V(x) - V(y) & \geq \epsilon, \quad \forall x \in \left\{ x \in \mathcal{X} \setminus \mathcal{X}_G \mid V(x) \leq \beta \right\}, \forall y \in \mathcal{B}_{\delta,p}(f(x,\pi(x)))
\label{eq:robust-decrease-cond} \\
V(x) &\geq \alpha, \quad \forall x \in \mathcal{X}_U. \label{eq:robust-safety-cond}
\end{align}
\end{definition}

This robust version strengthens the classical certificate by requiring the decrease condition to hold for all next states within a perturbation ball. Intuitively, the controller must still drive the system toward the goal and avoid unsafe regions despite small perturbations.
The value function \(V\) is therefore robustly decreasing until the goal is reached and remains outside unsafe states.
With this strengthened certificate condition, we can formally guarantee robust success of the trajectory:

\begin{theorem}[Robust RWA Guarantee]
\label{thm:robust-rwa-simple}
Let \( \{x_t\}_{t=0}^\infty \) be a trajectory generated by policy \( \pi \) and transition function \( f \), and the perturbation at each state is less than \(\delta\) under \(l_p\) norm , i.e., \( x_{t+1} \in \mathcal{B}_{\delta, p}(f(x_t,\pi(x_t))) \). If there exists a \( (\delta, p) \)-Robust Lyapunov Barrier Certificate \( V : \mathcal{X}_s \rightarrow \mathbb{R}_{\geq c}\) where \(\mathcal{X}_s = \left\{ x_\tau \mid \forall t \in [0, \tau-1]. x_t \not\in \mathcal{X}_G \right\} \subset \mathcal{X}\) with parameters \( (\alpha, \beta, \epsilon) \), satisfies condition Eq.~\eqref{eq:robust-init-cond}--\eqref{eq:robust-safety-cond}, then the trajectory:
    (i) will reach the goal region \( \mathcal{X}_G \) in finite time and
    (ii) will never enter the unsafe region \( \mathcal{X}_U \),
thereby completing the robust RWA task under state-wise perturbations.
\end{theorem}

The detailed proof is provided in Appendix~\ref{appendix:robust-rwa-proof}. Here \(\mathcal{X}_s\) denotes the set of all states in the trajectory before reaching the goal. At a high level, the guarantee follows from the fact that the robust Lyapunov condition ensures a decrease of \( V \) under bounded perturbations at each step. If any perturbed next state enters the unsafe set \( \mathcal{X}_U \), it would violate the safety condition enforced by the certificate Eq.~\eqref{eq:robust-safety-cond}. Similarly, if the trajectory failed to reach the target, \( V(x_t) \) would decrease indefinitely, contradicting the lower bound \(c\) of the certificate. Hence, satisfying the robust certificate conditions along the trajectory implies that the RWA task is completed under perturbations on states.

\subsection{Training Objective}
\label{sec:training-loss-design}
To synthesize a controller and certificate that satisfy the robust CLBF conditions under norm-bounded perturbations, we design a training objective with multiple loss terms. In this section, we describe how these loss components are formulated and how they relate to the robust CLBF condition. The total loss combines these terms and is later used in our training algorithm described in Appendix~\ref{appendix:training-alg}.

\paragraph{Initialization Loss:}
This loss term proposed in \cite{mandal2024formally} enforces \( V(x) \leq \beta \) for initial states \( x \in \mathcal{X}_I \) as required by condition Eq.~\eqref{eq:robust-init-cond}:
\begin{equation}
\label{eq:loss-init}
    \mathcal{L}_{\text{init}} = \sum_{x \in \mathcal{X}_I} \max(0, V(x) - \beta)
\end{equation}
When \( \mathcal{L}_{\text{init}} \) becomes \(0\), it means that \( V(x) -\beta \leq 0 \) for all data we sample from \( \mathcal{X}_I \), then Eq.~\eqref{eq:robust-init-cond} is satisfied in the training data.

\paragraph{Descent Loss:}
We first introduce the descent loss without any perturbation proposed in \cite{mandal2024formally}, which enforces \( V\) decrease at least \( \epsilon \) at each step before reaching the goal state as required by Eq.~\eqref{eq:decrease-cond}.
\begin{equation}
\label{eq:loss-dec}
    \mathcal{L}_{\text{dec}} =\sum_{x \in \mathcal{X} \setminus \mathcal{X}_G ~\land~ V(x) \leq \beta}  \max\left(0, \epsilon - (V(x) - V(f(x, \pi(x))))\right)
\end{equation}
When \( \mathcal{L}_{\text{dec}} \) becomes \(0\), \(\epsilon - (V(x) - V(f(x, \pi(x)))) \leq 0 \) for all data we sampled outside \( \mathcal{X}_G\) and filtered by the precondition \( V(x) \leq \beta \). Then, Eq.~\eqref{eq:decrease-cond} is satisfied in the training data.

\subsubsection{Enforcing Robust Lyapunov Condition}
The descent condition Eq.~\eqref{eq:loss-dec} ensures that the certificate value \( V(x) \) decreases at least \(\epsilon\) from the current state \( x \) to its next state \( f(x, \pi(x)) \). However, this ideal condition is not sufficient in practice, where modeling errors, sensor noise, or environmental disturbances can perturb the next state.
To address this, we introduce loss terms designed to enforce either empirical robustness or certified robustness. The key idea is to ensure that \( V(x) \) remains \(\epsilon\) larger than its value not only in the ideal next state but throughout a neighborhood around it, i.e. to satisfy Eq.~\eqref{eq:robust-decrease-cond} instead of Eq.~\eqref{eq:decrease-cond}. 
The loss terms described in the following each contribute to enforcing this property through different methods.
Let \( y_{\text{adv}} \) be the point in the \( l_p \)-ball with radius \(\delta\) around \( f(x, \pi(x)) \) that maximizes $V$:
\[
y_{\text{adv}} = \arg\max_{y \in \mathcal{B}_{\delta,p}(f(x, \pi(x)))} V(y).
\]
The robust descent condition Eq.~\eqref{eq:robust-decrease-cond} is then equivalent to requiring:
\(
V(x) - V(y_{\text{adv}}) \geq \epsilon.
\)
This equivalence follows from the fact that if \( V(x) \) is at least \( \epsilon \) greater than all values of \( V(y) \) within the neighborhood, then it must also be greater than the maximum of these values by at least \( \epsilon \), vice versa. This formula enables us to design loss terms based either on an empirical estimate of \( V(y_{\text{adv}}) \), obtained via adversarial optimization, or on certified upper bounds derived from the Lipschitz continuity of \( V \).

\paragraph{1. Adversarial Descent Loss.}
We first use adversarial attack strategies \cite{madry2017towards, goodfellow2014explaining, croce2020reliable} to find \(\hat{y}_\text{adv}\) to approximate \(y_\text{adv}\), i.e., $V(\hat{y}_\text{adv}) \leq V(y_\text{adv})$. We design a loss term for the robust decrease condition:
\begin{equation}
\label{eq:loss-dec-adv}
  \mathcal{L}_{\text{dec-adv}} =\sum_{\substack{x \in \mathcal{X} \setminus \mathcal{X}_G \land V(x) \leq \beta} } \max\left(0, \epsilon - (V(x) - V(\hat{y}_\text{adv})\right)  
\end{equation}
When this loss term comes to \(0\), it indicates that for every sampled state, the certificate value at the current state exceeds the approximated maximum value given by adversarial attack strategies within the perturbation neighborhood by at least \( \epsilon \), thus approximately satisfying the robust CLBF condition on the training data.
The total training loss becomes:
\[
\mathcal{L}_{\text{total}} = \lambda_{\text{init}} \mathcal{L}_{\text{init}} + \lambda_{\text{dec-adv}} \mathcal{L}_{\text{dec-adv}}
\]
where \(\lambda_{\text{init}}, \lambda_{\text{dec-adv}} \in \mathbb{R}\) are the coefficients for the loss terms respectively.

There is no guarantee that adversarial attack strategies always find the worst-case perturbation. As a result, \( V(\hat{y}_\text{adv}) \) may underestimate the greatest value of \(V\) in \(\mathcal{B}_{\delta,p}(f(x, \pi(x)))\). To overcome this limitation, we leverage the global Lipschitz constants of $V$. 
We introduce the following theorem connecting the Lipschitz constant with the robustness of the Lyapunov certificate:
\begin{theorem}
    \label{thm:lip-cert}
    Given a Lyapunov Barrier Certificate \( V \) defined following \cref{def:lyapunov-barrier-certificate} with parameters \((\alpha,\beta,\epsilon_r)\), if \(V\) is Lipschitz continuous under the \( l_p \) norm with global Lipschitz constant \( L_p \), and any perturbation in the \(x_{t+1}\) is in \(\mathcal{B}_{\delta,p}(f(x_t,\pi(x_t)))\), then the robust Lyapunov condition Eq.~\eqref{eq:robust-decrease-cond} holds with parameters \( (\alpha, \beta, \epsilon_r - L_p \cdot \delta)\).
\end{theorem}
\textit{Proof sketch.} The Lipschitz continuity of \(V\) implies that for any perturbed next state \(y\) in the \(\delta\)-ball around \(f(x,\pi(x))\), \(V(y)\) is at most \(L_p \cdot \delta\) larger than \(V(f(x, \pi(x)))\). Subtracting this from \(\epsilon_r\) ensures that the robust decrease condition holds. The full proof is provided in Appendix~\ref{appendix:lip-cert-proof}.

\paragraph{2. Lipschitz-bound Descent Loss.}
 Instead of empirically searching for \( y_\text{adv} \), we bound the maximum  of \( V(y) \) for \( y \in \mathcal{B}_{\delta,p}( f(x,\pi(x))) \) using global Lipschitz bound \( L_p\). This gives a modified descent loss:
\begin{equation}
\label{eq:loss-dec-neighbor}
    \mathcal{L}_{\text{dec-neighbor}} =\sum_{x \in \mathcal{X} \setminus \mathcal{X}_G \land V(x) \leq \beta}  \max\left(0, \epsilon - (V(x) - V(f(x,\pi(x))) - L_p \cdot \delta) \right) .
\end{equation}

\begin{theorem}
\label{thm:lip-neighbor-loss}
    If \( \mathcal{L}_\text{dec-neighbor} = 0 \), then \( Eq.~\eqref{eq:robust-decrease-cond} \) is satisfied on all sampled data.
\end{theorem}
\textit{Proof sketch.} When the loss is zero, we have \(V(x) - V(f(x,\pi(x))) \geq \epsilon + L_p \cdot \delta\). By \cref{thm:lip-cert}, this guarantees that \(V(x) - V(y) \geq \epsilon\) for all \(y \in \mathcal{B}_{\delta,p}(f(x,\pi(x)))\), thereby satisfying the robust condition. Full proof is in Appendix~\ref{appendix:lip-bound-proof}.

We do not explicitly constrain the global Lipschitz constant of the certificate function \(V\). Instead, we enforce a stronger descent condition by leveraging the connection between the Lipschitz continuity of a function and its local value differences. 
The total training loss becomes:
\[
\mathcal{L}_{\text{total}} = \lambda_{\text{init}} \mathcal{L}_{\text{init}} + \lambda_{\text{dec-neighbor}} \mathcal{L}_{\text{dec-neighbor}}
\]
where \(\lambda_{\text{init}}, \lambda_{\text{dec-neighbor}} \in \mathbb{R}\) are the coefficients for the loss terms, respectively.

\paragraph{3. Lipschitz Regularization Loss.}
\label{para:lip-reg-loss}
While the loss above provides robustness for training data, it does not provide certified guarantees for all the states in \( \mathcal{X}\). In safety-critical systems, however, we want robustness guarantees over the entire state space. To bridge this gap, we introduce the Lipschitz Regularization approach to achieve certifiable robustness on the full space. 
As given by Eq.~\eqref{eq:lip-matrix-norm}, we use the product of spectral norms of each layer to upper bound the global Lipschitz constant in \(l_2\) norm. To enforce an upper bound \( \tau \) on the global Lipschitz constant, we use:
\begin{equation}
\label{eq:loss-lip-global}
    \mathcal{L}_{\text{lip-global}} = \max\left(0, \prod_{k=1}^\ell \|W_k\|_2 - \tau\right).
\end{equation}

When this loss term reaches \(0\), we have \( \prod_{k=1}^\ell \|W_k\|_2 \leq \tau \). According to Eq.~\eqref{eq:lip-diff-norm}, the Lipschitz constant under the \( l_p \) norm is bounded by \(L_p = K_{p,2} \cdot \tau\),
where \( K_{p,2} \) is the norm conversion constant. From \cref{thm:lip-cert}, if the Lyapunov Barrier certificate \( V \) satisfies the standard conditions in \cref{def:lyapunov-barrier-certificate} and has its Lipschitz constant bounded by \( L_p \), then \( V \) also satisfies the robust Lyapunov Barrier Certificate definition in \cref{def:robust-lyapunov-barrier-certificate} with \( \epsilon = \epsilon_r - \delta \cdot L_p \). We aim to minimize the Lipschitz constant \( L_p \) because the robust condition \( \epsilon > 0 \) implies the constraint \( \delta < \epsilon_r / L_p \). This means that a smaller \( L_p \) allows for a larger admissible perturbation radius \( \delta \), thus improving the certified robustness. To determine the smallest achievable Lipschitz bound \( \tau \) that still satisfies the standard certificate conditions Eq.~\eqref{eq:init-cond}--\eqref{eq:safety-cond}, we perform a binary search of 
\(\tau\) between a lower bound of \(0\) and an upper bound given by the Lipschitz constant of a vanilla model trained to satisfy conditions the standard conditions. Since this vanilla model is already known to satisfy the standard conditions, the upper bound is guaranteed to be feasible. Since our algorithm ensures the Lipschitz bound \( \tau \) is met and the certificate conditions can be proved to be satisfied using the verifier when it finishes, we can guarantee that the resulting certificate is robust under perturbations of radius \( \delta \).

In the training process, we use the original training loss \(\mathcal{L}_\text{init}\) and \(\mathcal{L}_\text{dec}\) to satisfy the condition in Eq.~\eqref{eq:init-cond}--\eqref{eq:safety-cond}, and add the loss term \(\mathcal{L}_\text{lip-global}\) to upper bound the global Lipschitz constant. The total training loss becomes:
\[
\mathcal{L}_{\text{total}} = \lambda_{\text{init}} \mathcal{L}_{\text{init}} + \lambda_{\text{dec}} \mathcal{L}_{\text{dec}} + \lambda_{\text{lip-global}} \mathcal{L}_{\text{lip-global}}
\]
where \(\lambda_{\text{init}}, \lambda_{\text{dec}}, \lambda_\text{lip-global} \in \mathbb{R}\) are the coefficients for the loss terms, respectively. Since we have the loss to enforce Eq.~\eqref{eq:init-cond}--\eqref{eq:safety-cond} are satisfied, we do not need to explicitly enforce the robust decrease condition Eq.~\eqref{eq:robust-decrease-cond}. Instead, according to \cref{thm:lip-cert}, by bounding the Lipschitz constant and satisfying the nominal decrease condition, we ensure that the robust condition is also satisfied.

\section{Experiments}
In this section, we first describe the experimental setup. We then present results on both certified and empirical robustness, followed by an ablation study to analyze the impact of key hyperparameters. Each experiment is run multiple times, and detailed statistics are provided in Appendix~\ref{appendix:exp-results}.
\subsection{Experiment Setup}
We validate our method on two reinforcement learning environments: the \textit{2D Docking} task, which captures important safety-critical behaviors relevant to autonomous systems, and the \textit{Inverted Pendulum}, a popular benchmark in control theory. A detailed case study is provided in Appendix~\ref{appendix:case-study}. Following prior work in \cite{vzikelic2023learning, mandal2024formally}, We design the policy network for the 2D docking task with two hidden layers of 20 neurons each, and the policy for the inverted pendulum environment with two hidden layers of 128 neurons each, both using ReLU activations. The certificate network consists of three hidden layers with 64, 32, and 16 neurons respectively, also using ReLU activations.
All experiments were conducted on a server equipped with 12 cores of Intel Xeon Silver 4214R CPUs, 187 GB of RAM, and running Ubuntu 22.04. 

We set \( \beta = 1 \) in both environments, and mask the certificate value such that \( V(x) = -10 \) for goal states \( x \in \mathcal{X}_G \) and \( V(x) = 1.2 \) for unsafe states \( x \in \mathcal{X}_U \). The descent margin \( \epsilon \) is selected by training the non-robust Lyapunov Barrier Certificate (satisfying only Eq.~\eqref{eq:init-cond}--\eqref{eq:safety-cond}) with the largest possible margin until the certificate cannot be learned. We then keep this value fixed for robust training: \( \epsilon = 10^{-2} \) for 2D Docking and \( 5 \times 10^{-3} \) for Inverted Pendulum. For each method, we binary search for the largest Lipschitz bound \( \tau \) and neighborhood radius \( \delta \) that allow training to converge.
In the training process, we set \( \lambda_{\text{init}} = 1, \lambda_{\text{lip-global}}=1 \), while all other loss weights (i.e., \( \lambda_{\text{dec}} \), \( \lambda_{\text{dec-adv}} \), \( \lambda_{\text{dec-neighbor}} \)) are set to \(10\), depending on the specific training objective being used. We conducted our experiment on \(l_\infty\) norm. In the adversarial training method, we use Projected Gradient Descent (PGD)~\cite{madry2017towards} to approximate the value of \( y_{\text{adv}} \), the point in the neighborhood where \( V \) is maximized. 

In the CEGIS loop, we train until a loss of 0 is achieved and then use the verification step to find counterexamples. We repeat this until there are no more counterexamples or a timeout (12 hours) is reached. We use the Marabou verifier~\cite{wu2024marabou} in the verification step. When the system dynamics involve operations that cannot be precisely analyzed by existing DNN verifiers (e.g., trigonometric functions),  we encode piecewise-linear bounds that over-approximate the dynamics. The bounds are constructed with OVERT \cite{sidrane2022overt}, with the number of piecewise-linear segments set to 3 and an absolute approximation error tolerance set to \( \epsilon = 10^{-5} \). 

\subsection{Certified robustness} \label{subsec:certified}
\begin{table}[t]
\centering
\small
\renewcommand{\arraystretch}{1.2}
\begin{subtable}{0.48\textwidth}
\centering
\begin{adjustbox}{max width=\textwidth}
\begin{tabular}{@{}lr@{}}
    \toprule
    \textbf{Method} & \textbf{Certified Bound} \\
    \midrule
    Vanilla        & 0.0015 \\
    Lip-3          & \textbf{0.0084} \\
    Lip-Neighbor   & 0.0033 \\
    PGD            & 0.0078 \\
    \bottomrule
\end{tabular}
\end{adjustbox}
\caption{Inverted Pendulum environment}
\label{tab:pendulum-cert}
\end{subtable}
\hfill
\begin{subtable}{0.48\textwidth}
\centering
\begin{adjustbox}{max width=\textwidth}
\begin{tabular}{@{}lr@{}}
    \toprule
    \textbf{Method} & \textbf{Certified Bound} \\
    \midrule
    Vanilla        & 0.0094 \\
    Lip-1          & \textbf{0.0172} \\
    Lip-Neighbor   & 0.0156 \\
    PGD            & 0.0131 \\
    \bottomrule
\end{tabular}
\end{adjustbox}
\caption{2D Docking environment}
\label{tab:docking-cert}
\end{subtable}
\caption{Certified perturbation bounds for different training methods. \textbf{Method:} Vanilla denotes a baseline model trained using the standard method without robustness. Lip-\(n\) corresponds to training with Lipschitz bound \(\tau = n\) as in Eq.~\eqref{eq:loss-lip-global}. Lip-Neighbor uses a neighborhood margin of \( \delta = 0.01 \) (2D docking) or \( \delta = 10^{-3} \) (Inverted pendulum) as in Eq.~\eqref{eq:loss-dec-neighbor}. PGD is trained with adversarial examples in the perturbation range of \( \delta = 0.01 \) (2D docking) or \( \delta = 5 \times 10^{-3} \).}
\label{tab:certified-bounds}
\vspace{-7mm}
\end{table}
Certifiable robustness is evaluated using the Marabou DNN verifier~\cite{wu2024marabou} with a binary search procedure to find the largest \(\delta\) for which Eq.~\eqref{eq:robust-init-cond}--\eqref{eq:robust-safety-cond} hold with \(\epsilon = 10^{-6}\). We repeat the binary search procedure until the interval length is below \(10^{-4}\), and we then report the lower bound as the certified perturbation. Tables~\ref{tab:pendulum-cert} and~\ref{tab:docking-cert} show the certified bounds for the controller in Inverted Pendulum and 2D Docking environments, representing the maximum perturbation under which the controller provably complete in the robust RWA task. Compared to the baseline (Vanilla), our robust training methods improve certified robustness. In the Inverted Pendulum environment, Lip-3 improves the bound by 460\% (from 0.0015 to 0.0084), while in 2D Docking, Lip-1 increases it by 83\%. Since perturbations are measured under the \(l_\infty\) norm, certified volume grows exponentially with state dimension—Lip-3 increases the volume over 30× in 2D, and Lip-1 achieves over 10× in 4D.

\subsection{Empirical robustness} \label{subsec:empirical}
In addition to certifiable bounds, we also evaluate empirical robustness against both adversarial and random perturbations. After training, we simulate the controllers under adversarial perturbations and random perturbations. For adversarial perturbations, we apply attacks based on maximizing the certificate value in each step; for random perturbations, we apply bounded noise to the state computed by the environment at each time step. Tables~\ref{tab:pendulum-success} and~\ref{tab:docking-success} show the empirical success rates across perturbation ranges in both environments. A trajectory is considered successful if it reaches the goal region within 200 steps without entering the unsafe region. We report the success rate over 10,000 randomly sampled states that are in initial region but outside of the goal region. For each method, we train two models and report the average success rate.

In both environments, all proposed methods significantly outperform the Vanilla baseline. In the Inverted Pendulum environment, under a perturbation of $0.01$, Lip-3 and Lip-Neighbor achieve success rates of $70\%$ and $69\%$, nearly doubling the baseline's $39\%$, while PGD achieves the highest success at $89\%$. In the 2D Docking environment, both Lipschitz-based methods (Lip-1 and Lip-Neighbor) maintain $100\%$ success up to $0.03$ perturbation—$2.4\times$ better than Vanilla ($41\%$)—and retain strong performance at $0.05$ ($83\%$ and $86\%$). These comparisons clearly indicate that our methods not only exceed Vanilla performance within its certifiable perturbation radius, but also generalize substantially better under stronger perturbations.

PGD-trained models achieve high success rates within the perturbation radius they were trained on, but their performance significantly drops under larger perturbations. This highlights a well-known limitation of adversarial training that it does not generalize well to unseen, larger perturbations.
On the other hand, methods that incorporate Lipschitz constants show improved robustness across a wide range of perturbations. This is because these methods either implicitly (Lipschitz-Neighbor, bounding the maximum value using Lipschitz constant) or explicitly (Lip-n, upper-bounding the Lipschitz constant by n) enforce a smaller Lipschitz bound, encouraging smoother behavior in the neural network. As a result, small perturbations in the state lead to only minor changes in the certificate value. When the certificate value remains close to the value before perturbed, the decrease condition is more likely to hold, thereby improving the system's ability to remain safe and reach the goal state.

\begin{table}[t]
    \centering
    \renewcommand{\arraystretch}{1.2}
    
    \begin{subtable}[t]{0.48\textwidth}
        \centering
        \begin{adjustbox}{max width=\textwidth}
        \begin{tabular}{@{}lrrrr@{}}
            \toprule
            \textbf{Method} & \textbf{0.005 Adv} & \textbf{0.008 Adv} & \textbf{0.01 Adv} & \textbf{0.05 Rand} \\
            \midrule
            Vanilla         & 72\% & 41\% & 39\% & 100\% \\
            Lip-3           & \textbf{100\%} & \textbf{100\%} & 70\% & 100\% \\
            Lip-Neighbor    & 78\% & 68\% & 69\% & 100\% \\
            PGD             & \textbf{100\%} & \textbf{100\%} & \textbf{89\%} & 100\% \\
            \bottomrule
        \end{tabular}
        \end{adjustbox}
        \caption{Inverted Pendulum environment}
        \label{tab:pendulum-success}
    \end{subtable}
    \hfill
    \begin{subtable}[t]{0.48\textwidth}
        \centering
        \begin{adjustbox}{max width=\textwidth}
        \begin{tabular}{@{}lrrrr@{}}
            \toprule
            \textbf{Method} & \textbf{0.01 Adv} & \textbf{0.03 Adv} & \textbf{0.05 Adv} & \textbf{0.05 Rand} \\
            \midrule
            Vanilla         & 100\% & 41\% & 34\% & 100\% \\
            Lip-1           & 100\% & \textbf{100\%} & 83\% & 100\% \\
            Lip-Neighbor    & 100\% & \textbf{100\%} & \textbf{86\%} & 100\% \\
            PGD             & 100\% & 51\% & 17\% & 100\% \\
            \bottomrule
        \end{tabular}
        \end{adjustbox}
        \caption{2D Docking environment}
        \label{tab:docking-success}
    \end{subtable}
    
    \caption{Success rates (\%) under various adversarial and random perturbations in both environments. "$x$ Adv" indicates adversarial perturbations with $\ell_\infty$ norm bound $x$ applied using PGD attacks. "$x$ Rand" denotes uniformly random perturbations within the range applied at each time step.}
    \label{tab:empirical-success}
    \vspace{-5mm}
\end{table}

\subsection{Ablation Study}
We conduct an ablation study on the 2D docking setting to understand the influence of key hyperparameters on training performance and certified robustness. Full numerical results are in Appendix~\ref{appendix:ablation}.

\textbf{Effect of Robust Descent Margin \(\epsilon\).}
We vary \(\epsilon \in \{0.001, 0.005, 0.01\}\) in the Lipschitz Regularization method. We find that smaller margins lead to lower certified robustness with roughly the same training time, validating our approach of choosing the largest feasible \(\epsilon\) from the vanilla model.

\textbf{Effect of Lipschitz Upper Bound \(\tau\).}
We evaluate several values of \(\tau\) in the Lipschitz Regularization method. Lower \(\tau\) generally leads to improved certified robustness, consistent with the theoretical guarantee that a smaller Lipschitz constant allow certification under larger perturbations. However, setting \(\tau\) too small makes the learning process difficult to converge.

\textbf{Effect of Loss Weight \(\lambda_{\text{lip-global}}\).}
To study the impact of the Lipschitz loss weight in the Lipschitz Regularization method, we fix \(\lambda_{\text{init}} = 1\), \(\lambda_{\text{dec}} = 10\) and vary \(\lambda_{\text{lip}} \in \{0.5, 1, 5, 10, 20\}\). We observe that training time first decreases and then increases, and certified robustness first decreases and then increases. This suggests the balance between regularization strength and optimization dynamics.

\textbf{Effect of Counterexample Weighting \(w\).}
During CEGIS training, we test different loss weight ratios \(w\) between original and counterexample where \(w \in \{100,500,1000\}\). While giving higher weight to counterexamples helps training converge fast, setting the ratio too high reduces overall robustness, probably due to overfitting to the counterexamples.

\section{Conclusion and Limitations}
In this paper, we studied the problem of synthesizing robust neural Lyapunov barrier certificate, where the validity of the certificate holds even under norm-bounded time-varying perturbations in the system dynamics. We established the connection between the robustness of the certificate and Lipschitz continuity, and proposed practical training methods, including adversarial training and Lipschitz-based methods, to enforce robustness for synthesizing robust certificates. We showed that our approach improved both certifiable and empricial robustness in two environments: Inverted Pendulum, a popular benchmark for control, and 2D Docking, a safety-critical task relevant to autonomous systems.
A limitation of our work is its focus on robustness to perturbations in the states only. In practice, robustness may be required against a broader range of uncertainty, including perturbations in the system dynamics, observations, or deviations in the controller’s output. Addressing these sources of uncertainty is an important direction for future work, and could enhance the framework's effectiveness in real-world settings, where sim-to-real gaps often arise from such unmodeled variations.

\newpage
\bibliography{reference}

\newpage
\appendix
\section{Additional Related Work}
\label{appendix:related-work}

\subsection{Lipschitz Continuity}

Lipschitz continuity has been widely explored in the machine learning literature as a tool to improve model robustness, particularly in classification tasks. Many works use Lipschitz constant to partially explain the sensitivity of neural networks \cite{szegedy2013intriguing} or enforce global Lipschitz bound during training to improve robustness ~\cite{tsuzuku2018lipschitz,leino2021globally, cisse2017parseval}. However, these approaches are not directly applicable to control systems or Lyapunov-based verification, where the goal is not classification accuracy but certified state-space behavior over time. Our use of Lipschitz continuity focuses on certifying robustness under state perturbations in dynamic systems.

\subsection{Safe Reinforcement Learning}
Robust RL methods aim to improve policy performance under perturbations. Some approaches adopt the robust Markov Decision Process formulation\cite{wiesemann2013robust,lim2013reinforcement}, defining uncertainty sets around nominal dynamics~\cite{wang2023robust,clavier2022robust} and solve robust Bellman updates. A more widely adopted perspective models robustness as a minimax game against adversaries, first introduced in~\cite{morimoto2005robust}. Adversarial training frameworks build on this idea, employing single~\cite{pinto2017robust, tessler2019action}, population~\cite{vinitsky2020robust}, or risk-aware adversaries~\cite{pan2019risk} to simulate worst-case perturbations on policies and improve robustness. Other lines of work consider adversarial perturbations in the observation space~\cite{zhang2020robust, zhang2021robust, stanton2021robust}. While effective for improving empirical performance, these methods typically do not provide formal, trajectory-level robustness guarantees. In contrast, our approach certifies safety and stability by synthesizing neural Lyapunov-barrier functions that provably satisfy robustness under norm-bounded perturbations.

\section{Proof of Theorems}
\subsection{Norm Inequality Derivation}
\label{appendix:norm-inequality}

To generalize to other \( l_p \) norms, we use the following norm inequality:
\[
\|x\|_p \leq n^{\left(\frac{1}{p} - \frac{1}{q}\right)} \|x\|_q.
\]
Applying this to both the input and output spaces, we obtain a Lipschitz constant for the function \( f \) in \( l_q \)-input and \( l_p \)-output norms:
\begin{align}
    \|f(y) - f(x)\|_p 
    &\leq m^{\left(\frac{1}{p} - \frac{1}{2}\right)} \|f(y) - f(x)\|_2 \\
    &\leq m^{\left(\frac{1}{p} - \frac{1}{2}\right)} L_2 \|y - x\|_2 \\
    &\leq m^{\left(\frac{1}{p} - \frac{1}{2}\right)} L_2 \cdot n^{\left(\frac{1}{2} - \frac{1}{q}\right)} \|y - x\|_q,
\end{align}
where \( n \) and \( m \) are the input and output dimensions of \( f \) respectively, and \( L_2 \) is the global Lipschitz constant of \( f \) under the \( l_2 \) norm.
Thus, the overall bound becomes:
\begin{align}
    \|f(y) - f(x)\|_p 
    \leq n^{\left(\frac{1}{2} - \frac{1}{q}\right)} m^{\left(\frac{1}{p} - \frac{1}{2}\right)} L_2 \|y - x\|_q.
\end{align}

\subsection{Proof of \cref{thm:appendix-robust-rwa-simple}}
\label{appendix:robust-rwa-proof}
\begin{theorem}[Robust RWA Guarantee]
\label{thm:appendix-robust-rwa-simple}
Let \( \{x_t\}_{t=0}^\infty \) be a trajectory generated by policy \( \pi \) and transition function \( f \), and the perturbation at each state is less than \(\delta\) under \(l_p\) norm , i.e., \( x_{t+1} \in \mathcal{B}_{\delta, p}(f(x_t,\pi(x_t))) \). If there exists a \( (\delta, p) \)-Robust Lyapunov Barrier Certificate \( V : \mathcal{X}_s \rightarrow \mathbb{R}_{\geq c}\) where \(\mathcal{X}_s = \left\{ x_\tau \mid \forall t \in [0, \tau-1]. x_t \not\in \mathcal{X}_G \right\} \subset \mathcal{X}\) with parameters \( (\alpha, \beta, \epsilon) \), satisfies condition Eq.~\eqref{eq:robust-init-cond}--\eqref{eq:robust-safety-cond}, then the trajectory:
    (i) will reach the goal region \( \mathcal{X}_G \) in finite time and
    (ii) will never enter the unsafe region \( \mathcal{X}_U \),
thereby completing the robust RWA task under state-wise perturbations.
\end{theorem}

\begin{proof}
To prove (i), assume by contradiction that the trajectory never reaches \( \mathcal{X}_G \). Then \( x_t \in \mathcal{X}_s \) for all \( t \in \mathbb{N} \), and by Eq.~\eqref{eq:robust-decrease-cond} we have:
\[
V(x_{t}) - V(x_{t+1}) \geq \epsilon > 0.
\]
This implies \( V(x_t) \) strictly decreases by at least \( \epsilon_r \) at every step. Since \( V(x) \geq c \) for all \( x \in \mathcal{X}_s \), and \( V(x_0) \leq \beta \), the number of steps before reaching a contradiction is at most
\[
T \leq \frac{\beta - c}{\epsilon}.
\]
Therefore, the trajectory must reach \( \mathcal{X}_G \) in finite time.

To prove (ii), let \( \tau \) be the first time \( x_\tau \in \mathcal{X}_G \). For all \( t \leq \tau \), we have \( x_t \in \mathcal{X}_s \), and from Eq.~\eqref{eq:robust-decrease-cond}:
\[
V(x_t) < V(x_{t-1}) < \cdots < V(x_0) \leq \beta.
\]
Since \( \beta < \alpha \), it follows that \( V(x_t) < \alpha \), so \( x_t \notin \mathcal{X}_U \). Thus, the trajectory avoids the unsafe region at all times.
\end{proof}

\subsection{Proof of \cref{thm:appendix-lip-neighbor-loss}}
\label{appendix:lip-bound-proof}
\begin{theorem}
\label{thm:appendix-lip-neighbor-loss}
\( \mathcal{L}_\text{dec-neighbor} \) is defined by: 
\[
\mathcal{L}_{\text{dec-neighbor}} =\sum_{x \in \mathcal{X} \setminus \mathcal{X}_G \land V(x) \leq \beta}  \max\left(0, \epsilon - (V(x) - V(f(x,\pi(x))) - L_p \cdot \delta) \right) .
\]

If \( \mathcal{L}_\text{dec-neighbor} = 0 \), then \( Eq.~\eqref{eq:robust-decrease-cond} \) is satisfied on all sampled data.

\end{theorem}

\begin{proof}
Let \( x_{\text{next}} = f(x, \pi(x)) \). By the definition of the Lipschitz constant of \( V \) under the \( l_p \)-norm, we have

\[
\forall y \in \mathcal{B}_{\delta,p}(x_{\text{next}}),~
\|V(y) - V(x_{\text{next}})\|_p \leq L_p \cdot \|y - x_{\text{next}}\|_p \leq L_p \cdot \delta.
\]
Since \( V(y), V(x_{\text{next}}) \in \mathbb{R} \), we have:
\(
V(y) - V(x_\text{next}) \leq \|V(y) - V(x_{\text{next}})\|_p
\).
Thus,
\begin{equation}
\label{eq:proof-lip-neighbor}
     V(y) \leq V(x_{\text{next}}) - L_p \cdot \delta
\end{equation}
\begin{align*}
    \mathcal{L}_\text{dec-neighbor} = 0
    &\iff \forall x \in \mathcal{X} \setminus \mathcal{X}_G \land V(x) \leq \beta, \max(0, \epsilon - (V(x) - L_p \cdot \delta - V(x_\text{next}))) = 0 \\
    &\iff
    \forall x \in \left\{ x \in \mathcal{X} \setminus \mathcal{X}_G \mid V(x) \leq \beta \right\}, V(x) - L_p \cdot \delta - V(x_\text{next}) \geq \epsilon \\ 
\end{align*}

Applying Eq.~\eqref{eq:proof-lip-neighbor} we get:
\begin{equation*}
    \forall x \in \left\{ x \in \mathcal{X} \setminus \mathcal{X}_G \mid V(x) \leq \beta \right\}, \forall y, y \in \mathcal{B}_{\delta,p}(x_\text{next}), V(x) - V(y) \geq \epsilon,
\end{equation*}
    which is the same as Eq.~\eqref{eq:robust-decrease-cond}.
    
\end{proof}

\subsection{Proof of \cref{thm:appendix-lip-cert}}
\label{appendix:lip-cert-proof}
\begin{theorem}
    \label{thm:appendix-lip-cert}
    Given a Lyapunov Barrier Certificate \( V \) defined following \cref{def:lyapunov-barrier-certificate} with parameters \((\alpha,\beta,\epsilon_r)\), if \(V\) is Lipschitz continuous under the \( l_p \) norm with global Lipschitz constant \( L_p \), and any perturbation in the \(x_{t+1}\) is in \(\mathcal{B}_{\delta,p}(f(x_t,\pi(x_t)))\), then the robust Lyapunov condition Eq.~\eqref{eq:robust-decrease-cond} holds with parameters \( (\alpha, \beta, \epsilon_r - L_p \cdot \delta)\).
\end{theorem}
\begin{proof}
Let \(x_{\text{next}} = f(x, \pi(x))\) denote the unperturbed next state.
Suppose the perturbed next state is \( y \in \mathcal{B}_{\delta, p}(x_{\text{next}}) \), i.e.,
\[
\|y - x_{\text{next}}\|_p \leq \delta.
\]
By Lipschitz continuity of \( V \), we have:
\[
\|V(y) - V(x_{\text{next}})\|_p \leq L_p \cdot \|y - x_{\text{next}}\|_p \leq L_p \cdot \delta.
\]
From Eq.~\eqref{eq:decrease-cond} in \cref{def:lyapunov-barrier-certificate}:
\[
V(x_{\text{next}}) \leq V(x) - \epsilon_r,
\]
we get:
\[
V(y) \leq V(x_{\text{next}}) + L_p \cdot \delta \leq V(x) - \epsilon_r + L_p \cdot \delta.
\]
Rewriting this, we get:
\[
V(y) \leq V(x) - (\epsilon_r - L_p \cdot \delta),
\]
which proves that the robust decrease condition Eq.~\eqref{eq:robust-decrease-cond} holds under perturbations.

\end{proof}

\section{Case Study}
\label{appendix:case-study}
\paragraph{Inverted Pendulum}
The Inverted Pendulum environment is based on OpenAI Gym \cite{brockman2016openai}. The task is to drive the pendulum to the upright position with near-zero velocity and avoid falling down during the process. Specifically, the state of the system is \(\mathbf{x}=[\theta, \dot{\theta}]\), where \( \theta\) is the angle between the pendulum and the upright vertical, and \(\dot{\theta}\) is the angular velocity. The control input \(u \in [-1, 1]\) is the applied torque determined by the policy \(\pi(\mathbf{x})\). The dynamic system is:
\begin{align*}
    \dot\theta' &= (1-b) \dot\theta + (\frac{1.5 \cdot g \cdot \sin\theta}{2\cdot l} + \frac{3}{ml^2} \cdot 2u) \cdot T \\
    \theta' &= \theta + \dot\theta' \cdot T
\end{align*}
with \(g = 10\,\mathrm{m/s^2},~ m = 0.15\,\mathrm{kg},~ l = 0.5\,\mathrm{m},~ b = 0.1,~ T = 0.05\,\mathrm{s}\).
The definitions of state sets follow the previous work \cite{vzikelic2023learning} with \(\mathcal{X}=[-0.7,0.7]^2, \mathcal{X}_I=[-0.3,0.3]^2, \mathcal{X}_G=[-0.2,0.2]^2, \mathcal{X}_U=[-0.7, -0.6] \times [-0.7, 0] \cup [0.6, 0.7] \times [0, 0.7]\).
\paragraph{2D Docking}
The 2D Docking task is taken from \cite{ravaioli2022safe}. Specifically, the task is to drive the deputy spacecraft, controlled with thrusters that provide forces in the \(x\) and \(y\) directions, to safely reach the goal state that is in close proximity to the chief spacecraft, while avoiding entering unsafe states. The state is \(\mathbf{x} = [x, y, v_x, v_y]\), where \((x, y)\) is position and \((v_x, v_y)\) is velocity. The control input \(\mathbf{u} = [F_x, F_y] \in [-1, 1]^2\) are the thrust forces applied along the \(x\) and \(y\) directions respectively, which is given by policy \(\pi(\mathbf{x})\).
The dynamic system given by transition function is:
\begin{align}
x' &= \left( \frac{2 v_y}{n} + 4 x + \frac{F_x}{m n^2} \right) + \left( \frac{2 F_y}{mn} \right) \nonumber \\
&\quad + \left( - \frac{F_x}{m n^2} - \frac{2 v_y}{n} - 3 x \right) \cos(nT) 
+ \left( \frac{-2 F_y}{m n^2} + \frac{v_x}{n} \right) \sin(nT) \label{eq:xdot} \\
y' &= \left( -\frac{2v_x}{n} + y + \frac{4 F_y}{m n^2} \right) 
+ \left( \frac{-2 F_x}{mn} - 3v_y - 6n x \right) T 
- \frac{3F_y}{2m} t^2 \nonumber \\
&\quad + \left( -\frac{4F_y}{mn^2} + \frac{2v_x}{n} \right) \cos(nT)
+ \left( \frac{2F_x}{mn^2} + \frac{4v_y}{n} + 6x \right) \sin(nT) \label{eq:ydot} \\
v_x' &= \left( \frac{2F_x}{mn} \right) 
+ \left( \frac{-2F_y}{mn} + x \right) \cos(nT) 
+ \left( \frac{F_x}{mn} + 2v_y + 3n x \right) \sin(nT) \label{eq:vxprime} \\
v_y' &= \left( \frac{-2F_x}{mn} - 3v_y - 6n x \right)
+ \left( -\frac{3 F_y}{m} \right) T \nonumber \\
&\quad + \left( \frac{2F_x}{mn} + 4v_y + 6n x \right) \cos(nT)
+ \left( \frac{4 F_y}{mn} - 2 v_x \right) \sin(nT) \label{eq:vyprime}
\end{align}

with \( m = 12\,\mathrm{kg}, ~ n = 0.001027\,\mathrm{rad/s}, ~ T = 1\,\mathrm{s}\), and \(F_x, F_y\) between -1 and +1 Newtons. \(\mathbf{x} = [x,y, v_x, v_y]\) is the current state and \(\mathbf{x'} = [x',y', v_x', v_y']\) is the next state given by transition function \(f\). Following \cite{mandal2024formally}, the state sets are defined as \( 
\mathcal{X}=\mathbb{R}^4, ~\mathcal{X}_{safe}=[-2,2]^2 \times [-0.5,0.5]^2, ~\mathcal{X}_U=\mathcal{X} \setminus \mathcal{X}_{safe}, ~\mathcal{X}_G = [-0.35, 0.35]^2 \times \mathbb{R}^2, ~\mathcal{X}_I = [-1,1]^2 \times [0,0]^2\).

\section{Training Algorithm}

Concretely, we sample data from the state space and use loss functions to penalize the violation of the condition in Eq.~\eqref{eq:robust-init-cond}--\eqref{eq:robust-safety-cond}. Note that \(\mathcal{X}_U \cap \mathcal{X}_I = \emptyset  \) and \(\forall x \in \mathcal{X}_U,  V(x) > \beta\), so \( x \in \mathcal{X}_U\) only influences Eq.~\eqref{eq:robust-safety-cond}. For simplicity of training, we do not add this constraint during training. Instead, we output a filtered certificate \cite{mandal2024formally} on the certificate we get from the training, in which all the states \(x \in \mathcal{X}_U\) are assigned to \(e\) where \(e > \alpha\). The filtered certificate \(V\) then satisfies Eq.~\eqref{eq:robust-safety-cond} for all $x \in \mathcal{X}_U$, so we do not need sample from \( \mathcal{X}_U \). There is no constraint for \( x \in \mathcal{X}_G\), so we also do not need sample data from \( \mathcal{X}_G \). We only need to consider Eq.~\eqref{eq:robust-init-cond} and Eq.~\eqref{eq:robust-decrease-cond}.

We first train an initial controller policy \( \pi\) via a reinforcement learning algorithm. Next, we perform several training epochs in which only the parameters of the certificate function \(V\) are updated to initialize it, before jointly training both the controller and the certificate. Then in the CEGIS loop, every iteration we first train the controller policy \(\pi\) and certificate \(V\) together to minimize the total loss \( \mathcal{L}_\text{total}\), which enforces the controller and certificate to satisfy Eq.~\eqref{eq:robust-init-cond} and Eq.~\eqref{eq:robust-decrease-cond}. We describe the design of loss functions in Section \ref{sec:training-loss-design}. Once the loss reaches 0, it means that all the training data satisfy these conditions, then we use a neural network verifier to search for counterexamples that violate these conditions. If such counterexamples are not found, we get the verified controller and certificate. Otherwise, we augment the training dataset with samples from their neighborhoods and retrain the controller and certificate. This local resampling is motivated by the Lipschitz continuity of both the controller and the certificate: If a violation occurs at one state, nearby states are also likely to violate the conditions.

\label{appendix:training-alg}
\begin{algorithm}[H]
\caption{CEGIS Loop for Robust Lyapunov Barrier Certificate}
\label{alg:cegis-loop}
\begin{algorithmic}[1]
\State \textbf{Input:} 
    \State \quad Environment dynamics $f$
    \State \quad State sets $\mathcal{X}_I \leftarrow \text{initial states}$, $\mathcal{X}_G \leftarrow \text{goal states}$, $\mathcal{X}_U \leftarrow \text{unsafe states}$
    \State \quad $\mathcal{D}_\text{orig} \leftarrow \text{training data sampled from } \mathcal{X} \backslash \{\mathcal{X}_G \cup \mathcal{X}_U\} $ 
    \State \quad $\beta \leftarrow \text{initial certificate threshold}$
    \State \quad $\epsilon \leftarrow \text{descent margin}$
    \State \quad $ w \leftarrow \text{loss weight for counterexample data}$
    \State \quad $ \lambda_* = \{\lambda_\text{init},\lambda_\text{dec},\lambda_\text{dec-adv}, \lambda_\text{dec-neighbor},\lambda_\text{lip-global}\} \leftarrow \text{training loss weights}$
    \State \quad $ \tau \leftarrow \text{upper bound of Lipschitz constant, for Lipschitz Regularization method}$
    \State \quad $\text{Verifier} \leftarrow \text{neural network verifier}$
    
\State \textbf{Initialize:}
    \State \quad $\pi \leftarrow \text{trained by RL algorithm}$
    \State \quad $V \leftarrow \text{trained by minimizing } \mathcal{L}_{\text{total}}(\beta,\epsilon,\tau,\lambda_*)$ on dataset \(\mathcal{D}_\text{orig}\)
    \State \quad $\mathcal{D}_\text{ce} \leftarrow \emptyset$
\Repeat
    \State \textbf{Training:}
        \State \quad $\mathcal{L}_\text{orig} \leftarrow \mathcal{L}_\text{total}(\beta,\epsilon,\tau,\lambda_*) \text{ calculated on dataset } \mathcal{D}_\text{orig} $ 
        \State \quad $\mathcal{L}_\text{ce} \leftarrow \mathcal{L}_\text{total}(\beta,\epsilon,\tau,\lambda_*) \text{ calculated on dataset } \mathcal{D}_\text{ce} $ if $\mathcal{D}_\text{ce} \not= \emptyset $ else 0
        \State \quad $\mathcal{L}_\text{final} \leftarrow \mathcal{L}_\text{orig} + w\cdot\mathcal{L}_\text{ce} $ 
        \State \quad $(V, \pi) \leftarrow \text{trained by minimizing } \mathcal{L}_{\text{final}} $

    \State \textbf{Verification:}
        \State \quad $\text{counterexamples} \leftarrow \text{Verifier}(V_\theta, \pi_\phi)$
        
    \If{$\text{counterexamples} \neq \emptyset$}
        \State \textbf{Refinement:}
            \State \quad $\mathcal{D}_{\text{add}} \leftarrow \text{uniformly sample \(m\) data points in \(l_p\) ball around each counterexample}$
            \State \quad $\mathcal{D}_\text{ce} \leftarrow \mathcal{D}_\text{ce} \cup \mathcal{D}_{\text{add}}$
    \Else
        \State \textbf{break}
    \EndIf
\Until{$\text{max iterations reached} \lor \text{counterexamples} = \emptyset$}

\State \textbf{Output:} $(V, \pi) \leftarrow \text{certified certificate-controller pair}$
\end{algorithmic}
\end{algorithm}

\section{Supplementary Experiments}
\subsection{Experiment Results}
\label{appendix:exp-results}
We re-ran our experiments with 5 trials. We report the average performance and standard deviation, visualized with error bars in the accompanying figures. 
Table~\ref{tab:certified-bounds-multiruns} summarize the certified robustness across different methods in both environments. Figure~\ref{fig:2d-docking-multiruns-pgd} and Figure~\ref{fig:inverted-pendulum-multiruns-pgd} illustrate the empirical robustness of each method under PGD attacks at various perturbations in both environments. 
We observe that the overall trends and conclusions remain consistent with Section~\ref{subsec:certified} and Section~\ref{subsec:empirical}.

\begin{table}[t]
\centering

\begin{subtable}{0.48\textwidth}
\centering
\begin{adjustbox}{max width=\textwidth}
\begin{tabular}{@{}lr@{}}
    \toprule
    \textbf{Method} & \textbf{Certified Bound (Mean $\pm$ Std)} \\
    \midrule
    Vanilla        & 0.0042 $\pm$ 0.0012 \\
    Lip-3          & 0.0072 $\pm$ 0.0017 \\
    Lip-Neighbor   & 0.0062 $\pm$ 0.0022 \\
    PGD            & 0.0081 $\pm$ 0.0004 \\
    \bottomrule
\end{tabular}
\end{adjustbox}
\caption{Inverted Pendulum}
\end{subtable}
\hfill
\begin{subtable}{0.48\textwidth}
\centering
\begin{adjustbox}{max width=\textwidth}
\begin{tabular}{@{}lr@{}}
    \toprule
    \textbf{Method} & \textbf{Certified Bound (Mean $\pm$ Std)} \\
    \midrule
    Vanilla        & 0.0086 $\pm$ 0.0009 \\
    Lip-1          & 0.0163 $\pm$ 0.0012 \\
    Lip-Neighbor   & 0.0154 $\pm$ 0.0016 \\
    PGD            & 0.0124 $\pm$ 0.0010 \\
    \bottomrule
\end{tabular}
\end{adjustbox}
\caption{2D Docking} 
\end{subtable}

\caption{
Certified perturbation bounds (mean $\pm$ std) for different training methods.}
\label{tab:certified-bounds-multiruns}
\end{table}

\begin{figure}
    \centering
    \includegraphics[width=0.75\linewidth]{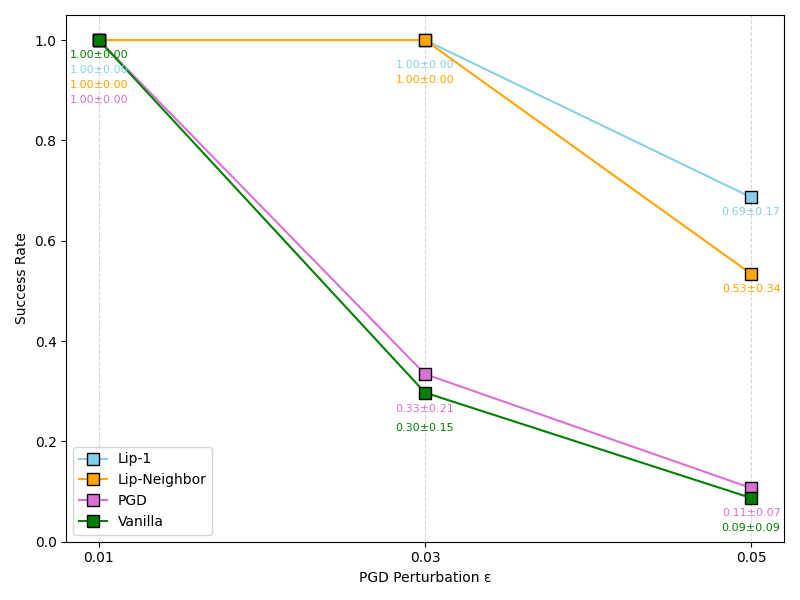}
    \caption{Success rate under various PGD perturbations in 2D docking}
    \label{fig:2d-docking-multiruns-pgd}
\end{figure}

\begin{figure}
    \centering
    \includegraphics[width=0.75\linewidth]{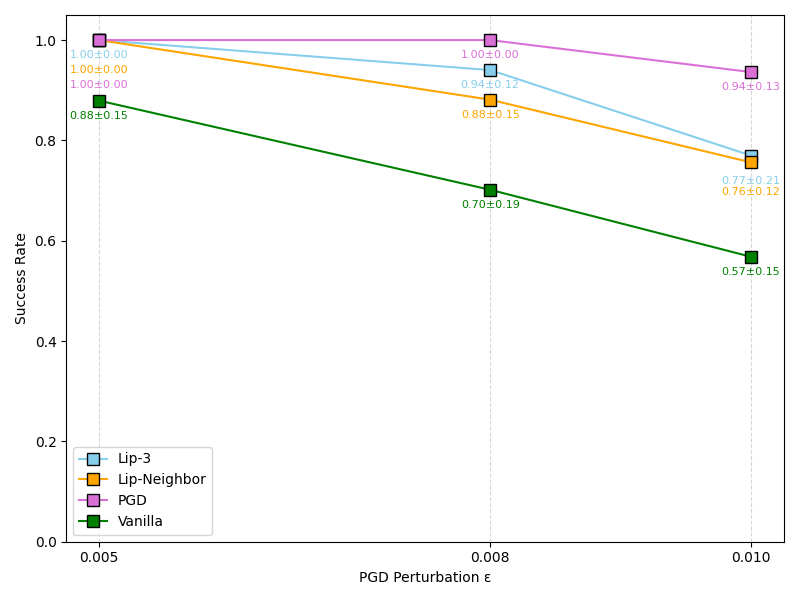}
    \caption{Success rate under various PGD perturbations in Inverted Pendulum}
    \label{fig:inverted-pendulum-multiruns-pgd}
\end{figure}

\subsection{Ablation Study}
\label{appendix:ablation}
In Table~\ref{tab:combined-ablation} we summarize the effects of key training hyperparameters on robustness and convergence. For each setting, we report:
\begin{itemize}
    \item the number of CEGIS loop iterations, which reflects the difficulty of achieving convergence;
    \item the certified perturbation bound, representing the model’s certifiable robustness;
    \item and empirical robustness, shown as the percentage of trajectories successfully reaching the goal under PGD attacks at different perturbation levels (0.01, 0.03, and 0.05).
\end{itemize}
\begin{table}[h]
\centering
\begin{tabular}{@{}llrrrrr@{}}
\toprule
Hyperparameter & Value & CEGIS Iterations & Certified Bound & 0.01 adv & 0.03 adv & 0.05 adv \\
\midrule

\multirow{3}{*}{$w$} 
& 1000 & 4 & 0.014 & 100\% & 100\% & 23\% \\
& 500  & 5 & 0.013 & 100\% & 100\% & 35\% \\
& 100  & 4 & 0.013 & 100\% & 100\% & 45\% \\

\midrule

\multirow{4}{*}{$\tau$} 
& 0.8\textsuperscript{*} & -- & -- & -- & -- & -- \\
& 1   & 3 & 0.015 & 100\% & 100\% & 85\% \\
& 2   & 6 & 0.013 & 100\% & 100\% & 45\% \\
& 3   & 6 & 0.013 & 100\% & 100\% & 54\% \\

\midrule

\multirow{5}{*}{$\lambda_{\text{global-lip}}$} 
& 0.5  & 8 & 0.013 & 100\% & 83\% & 16\% \\
& 1    & 4 & 0.013 & 100\% & 100\% & 45\% \\
& 5    & 6 & 0.010 & 100\% & 100\% & 31\% \\
& 10   & 8 & 0.015 & 100\% & 100\% & 47\% \\
& 20   & 8 & 0.015 & 100\% & 100\% & 22\% \\

\midrule

\multirow{3}{*}{$\epsilon$} 
& 0.01  & 4 & 0.012 & 100\% & 100\% & 45\% \\
& 0.005 & 9 & 0.010 & 100\% & 83\%  & 9.7\% \\
& 0.001 & 7 & 0.003 & 100\% & 10.9\% & 12.1\% \\

\bottomrule
\end{tabular}
\caption{Effects of varying key hyperparameters on convergence and robustness. \textsuperscript{*}Training with \(\tau = 0.8\) did not converge within 12 hours.}
\label{tab:combined-ablation}
\end{table}



\end{document}